\documentclass[letterpaper, 10 pt, conference]{ieeeconf}  

\IEEEoverridecommandlockouts                              

\usepackage{amsmath,amssymb,amsfonts,amsthm,dsfont} 
\usepackage{graphics}
\usepackage[dvips]{graphicx}
\usepackage[table,usenames,dvipsnames]{xcolor}
\usepackage{subcaption}
\usepackage{tikz-cd}
\usepackage{algorithm2e}
\usepackage{accents}
\usepackage[breaklinks=true, colorlinks, bookmarks=true, citecolor=Black, urlcolor=Violet,linkcolor=Black]{hyperref}

\newtheorem{proposition}{Proposition}

\theoremstyle{definition}
\newtheorem{definition}{Definition}
\newtheorem*{problem*}{Problem}

\newcommand{\pa}{\partial}
\newcommand{\ba}{\begin{align}}
\newcommand{\ea}{\end{align}}
\newcommand{\fr}{\frac}

\newcommand{\R}{{\mathbb R}}
\newcommand{\E}{{\mathbb E}}

\newcommand{\N}{{\mathcal N}}
\DeclareMathOperator*{\tr}{tr}
\DeclareMathOperator*{\diag}{diag}
\DeclareMathOperator*{\vech}{vech}
\DeclareMathOperator*{\argmin}{argmin}
\newcommand{\sinc}{\textrm{sinc}}
\newcommand{\eps}{\varepsilon}

\newcommand{\calA}{{\cal A}}
\newcommand{\calB}{{\cal B}}

\newcommand{\calD}{{\cal D}}

\newcommand{\calF}{{\cal F}}

\newcommand{\calI}{{\cal I}}

\newcommand{\calQ}{{\cal Q}}

\newcommand{\bfb}{\mathbf{b}}

\newcommand{\bfd}{\mathbf{d}}
\newcommand{\bfe}{\mathbf{e}}
\newcommand{\bff}{\mathbf{f}}
\newcommand{\bfg}{\mathbf{g}}
\newcommand{\bfh}{\mathbf{h}}

\newcommand{\bfm}{\mathbf{m}}

\newcommand{\bfp}{\mathbf{p}}
\newcommand{\bfq}{\mathbf{q}}
\newcommand{\bfr}{\mathbf{r}}
\newcommand{\bfs}{\mathbf{s}}

\newcommand{\bfu}{\mathbf{u}}
\newcommand{\bfv}{\mathbf{v}}
\newcommand{\bfw}{\mathbf{w}}
\newcommand{\bfx}{\mathbf{x}}
\newcommand{\bfy}{\mathbf{y}}
\newcommand{\bfz}{\mathbf{z}}

\newcommand{\bfalpha}{\boldsymbol{\alpha}}

\newcommand{\bfeta}{\boldsymbol{\eta}}

\newcommand{\bfpi}{\boldsymbol{\pi}}

\newcommand{\bfsigma}{\boldsymbol{\sigma}}

\newcommand{\bfxi}{\boldsymbol{\xi}}

\newcommand{\bfU}{\mathbf{U}}

\newcommand{\bbH}{\mathbb{H}}

\newcommand{\bbN}{\mathbb{N}}

\newcommand{\bbR}{\mathbb{R}}
\newcommand{\bbS}{\mathbb{S}}

\title{\LARGE \bf Active SLAM over Continuous Trajectory and Control: \\ A Covariance-Feedback Approach}

\author{Shumon Koga \and Arash Asgharivaskasi \and Nikolay Atanasov
\thanks{We gratefully acknowledge support from ARL DCIST CRA W911NF-17-2-0181 and NSF FRR CAREER 2045945.}%
\thanks{The authors are with the Department of Electrical and Computer Engineering, UC San Diego, 9500 Gilman Drive, La Jolla, CA, 92093-0411, {\tt\small \{skoga,aasghari,natanasov\}@ucsd.edu}.}
}

\begin{document}

\maketitle
\thispagestyle{empty}
\pagestyle{empty}


\begin{abstract} 

 This paper proposes a novel active Simultaneous Localization and Mapping (SLAM) method with continuous trajectory optimization over a stochastic robot dynamics model. The problem is formalized as a stochastic optimal control over the continuous robot kinematic model 
 to minimize a cost function that involves the covariance matrix of the landmark states. 
 We tackle the problem by separately obtaining an open-loop control sequence subject to deterministic dynamics by iterative Covariance Regulation (iCR) and a closed-loop feedback control under stochastic robot and covariance dynamics by Linear Quadratic Regulator (LQR). 
 The proposed optimization method 
 captures the coupling between localization and mapping in predicting uncertainty evolution and synthesizes highly informative sensing trajectories. 
We demonstrate its performance in active landmark-based SLAM using relative-position measurements with a limited field of view.
\end{abstract}


\section{Introduction}

Simultaneous Localization and Mapping (SLAM) has been instrumental for enabling autonomous robots to transition from controlled, structured, and fully known environments to operation in a priori unknown real-world conditions \cite{cadena2016past,rosen2021advances}. Many current SLAM techniques, however, remain \emph{passive} in their utilization of sensor data. Active SLAM \cite{delmerico2017active} is an extension of the SLAM problem which couples perception and control, aiming to acquire more information about the environment and reduce the uncertainty in the localization and mapping process. 
Active SLAM introduces unique challenges related to keeping the map and location estimation processes accurate, and yet computing and propagating uncertainty over many potential sensing trajectories efficiently to select an informative one. 

Many existing works in active perception decouple the localization and mapping problems and assume known robot states when planning active mapping trajectories. The literature on active mapping can be categorized according to the map representation it employs. Some techniques use volumetric mapping, which represents occupancy (e.g., occupancy grid) or obstacle distance (e.g., signed distance field) at a finite number of voxels obtained by discretizing the environment. Other techniques employ landmark-based mapping, which represents positions of a finite number of landmarks (e.g., objects or visual features) in the environment. While a volumetric representation captures the complete geometric structure of the environment, landmark-based mapping requires much less memory. One of the earliest approaches for active mapping \cite{yamauchi1998frontier} is based on detecting and planning a shortest path to frontiers (boundaries between explored and unexplored space) in a volumetric map.
Information-theoretic planning is an alternative approach, which utilizes an information measure to quantify and minimize the uncertainty in the map, as developed first in \cite{elfes1995robot} and subsequently used widely in robotics \cite{moorehead2001autonomous, grocholsky2002information, hollinger2014sampling}.  
Efficient computation methods for computing uncertainty of volumetric maps have been proposed in \cite{charrow2015information} for Cauchy-Schwarz quadratic mutual information (CSQMI), and in \cite{zhang2020fsmi} for fast Shannon mutual information (FSMI). Active mapping for truncated signed distance field (TSDF) reconstruction has been considered in \cite{saulnier2020information} and multi-category semantic maps have been studied in \cite{asgharivaskasi2021active}. Existing methods
are, however, limited to discrete control spaces, typically with a finite number of possible control inputs, such as \cite{saulnier2020information, charrow2015information, zhang2020fsmi}. 
Recently, in \cite{koga2021active}, we have developed a continuous trajectory optimization method for active mapping, named iterative Covariance Regulation (iCR). We have introduced a differentiable field of view in sensing model, and apply gradient descent method for obtaining an open-loop control sequence to maximize the differential entropy of the map. 


In applications other than mapping, motion planning under uncertainty in the robot states has been developed for several robotics tasks, such as reaching a goal without collisions with obstacles.  
To cope with the uncertainty in the motion model, the probability density function of the robot state given a sensory data is constructed, named as \emph{belief space}. An important work in this area is Belief Roadmaps \cite{prentice2009belief}, which projects a roadmap from the state space to the belief space and seeks an optimal path in the constructed graph.
Such a sampling-based method has been further developed in \cite{agha2011firm,agha2018slap}. 
Alternatively, continuous-space optimization methods have also been proposed for belief space planning. Representative work by \cite{van2012motion} applies an iterative LQG \cite{todorov2005generalized}, which computes both a nominal open-loop trajectory and a feedback control policy through iterative solutions of a dynamic programming. The authors have utilized iLQG for the belief dynamics consisting of the mean and the covariance of the robot state through EKF estimate utilizing a sensory data. However, the method in \cite{van2012motion} limits the observation model to be smooth for enabling the gradient computation in iLQG, while a typical measurement by a camera field of view does not follow such a smooth observation model. Recently, \cite{rahman2021uncertainty} has relaxed the assumption by introducing a probabilistic visibility model in sensing, and proposed a novel motion planning method with guaranteeing a constraint on the uncertainty in the robot state to be satisfied, via employing an augmented Lagrangian method. This paper differs from the belief space planning proposed in \cite{van2012motion,rahman2021uncertainty} in the sense that we take into account a target dynamics as a belief state and propose a method of solving an open-loop trajectory and closed-loop policy \emph{separately}, which is computationally efficient since the iteration is needed only for an open-loop trajectory under a deterministic dynamics.


Active SLAM is a challenging problem due to the mutual dependence among the accuracy of the robot localization, the performance of the mapping, and the trade-off between exploring new areas and exploiting uncertainty minimization in visited areas. Several techniques in the literature approach the active SLAM problem with a landmark-based mapping using greedy planning over a discrete control space, see \cite{carlone2014active} for instance. To avoid the costly global planning with long time horizons, in \cite{leung2006active}, the authors have introduced an attractor in landmark-based active SLAM, which incorporates global information about the environment for the local planner, and applied a model predictive control approach. Utilizing the idea of introducing the attractor, in \cite{atanasov2015decentralized}, a decentralized nonmyopic approach to a multi-robot landmark-based active SLAM has been proposed via exploiting the sparsity in the information filter. In \cite{Kantaros_InformationGathering_RSS19}, a multi-robot landmark-based active SLAM has been tackled by develoing a scalable sampling-based planning. 
While almost all of the literature in active SLAM have employed a discretized control space planning, \cite{martinez2009bayesian} has developed an online path planning method for active SLAM under a continuous control space by a Beyesian optimization for updating a parameter in a control policy. However, the method requires a sampling of the state propagation to approximate a mean square error of SLAM and an iteration of policy search, which renders a difficulty in the real-time implementation with an online planning. 

Unlike existing active SLAM techniques which commonly consider a discretized control space and known robot poses, 
this work develops an active SLAM method with continuous trajectory optimization over a stochastic robot dynamics model utilizing offline planning. A major advantage of the proposed method is its ability to capture the coupling between localization and mapping in predicting uncertainty evolution and to synthesize highly informative sensing trajectories due to the continuous-space optimization, which is computationally efficient and capable of real-time implementation through an offline computation.
We first provide a general formulation of the active information acquisition as studied in \cite{Atanasov14ICRA, saulnier2020information, koga2021active}. Apart from the previous literature, we include a stochastic process noise in the robot dynamics, which needs to be dealt with in application to active SLAM. Next, we propose a method for obtaining a nominal open-loop trajectory via iCR for deterministic robot dynamics and a closed-loop control policy by LQR for a linearized stochastic system around the iCR trajectory. Then, we apply these techniques to active SLAM to estimate the positions of a finite number of landmarks. 




\section{Problem Statement} \label{sec:problem} 


We consider a sensing system with state $\bfx_k \in \R^{n_x}$ and control input $\bfu_k \in \R^m$ at time $t_k \in \R_+$ where $\{t_k\}_{k=0}^K$ for some $K \in \bbN$ is an increasing sequence. The task is to track and detect a target state, denoted as $\bfy_k \in \R^{n_y}$, through on-board sensors which receives a sensor measurement $\bfz_k \in \R^{n_z}$ as a function of both the sensing system and target states.  We consider a stochastic nonlinear dynamics of the system state, a linear stochastic dynamics of the target state, and a linear observation model with respect to the target state, described by
\begin{equation}\begin{aligned}
    \bfx_{k+1} &= \bff(\bfx_k, \bfu_k, \bfw_k),   \\
    \bfy_{k+1} &= A \bfy_k + \Xi_k^{1/2} \bfxi_k , \\
    \bfz_k &= H(\bfx_k) \bfy_k + V(\bfx_{k})^{1/2} \bfv_k,  \label{eq:AIA1}
\end{aligned}\end{equation}
where $\bfw_k \sim \N(0, W_k)$, $\bfxi_k \sim \N(0, I_{n_y})$, $\bfv_k \sim \N(0, I_{n_z})$, $W_k \in \bbS_{\succ 0}^{n_x \times n_x}$, $\Xi_k \in \bbS_{\succ 0}^{n_y \times n_y}$, $V: \R^{n_x} \to \bbS_{\succ 0}^{n_z \times n_z}$, and $H: \R^{n_x} \to \R^{n_z \times n_y}$. $\bbS_{\succ 0}^{n \times n}$ is a set of positive definite matrices within $\R^{n \times n}$. 

Active information acquisition is a motion planning problem for a system dynamics aiming to minimize some uncertainty measure of the target state. One representative candidate for the measure is the differential entropy in the target state conditioned on the sensor states and measurements: 
\begin{align} \label{eq:cost-entropy} 
    \min_{\bfu_0, \dots, \bfu_{K-1}} \bbH (\bfy_K | \bfz_{1:K}, \bfx_{1:K} ). 
\end{align}
In \cite{le2009trajectory}, for deterministic system dynamics (i.e., $\bfw_k = 0$),  
the problem of minimizing \eqref{eq:cost-entropy} is shown to be equivalent to minimizing the log determinant of the covariance matrix of the target state as a deterministic optimal control, which is known as $D$-optimality in optimal experimental design \cite{optimalExpDesign}. 
There are several other criteria in optimal experimental design, such as $A$-optimality, minimizing the trace of the covariance, and $E$-optimality, minimizing the maximum eigenvalue \cite{carrillo2012comparison}. Here, we consider a general cost function over the covariance matrix to capture all possible optimality criteria. 

We approach the problem of minimizing an uncertainty criterion over the target state subject to motion and sensor models in \eqref{eq:AIA1} by extending the formulation in \cite{Atanasov14ICRA} to include stochastic noise in the robot dynamics. Let $\Sigma_k \in \bbS_{\succ 0}^{n_y \times n_y}$ be the covariance matrix of the target state. Then, the optimization problem to minimize \eqref{eq:cost-entropy} subject to \eqref{eq:AIA1} is rewritten as:
\begin{align}
\min_{\bfU \in \R^{mK}} \E& \left\{ \sum_{k=0}^{K-1} \bar c_k \left(\Sigma_k\right) + \bar c_K \left( \Sigma_{K} \right) \right\}, \label{eq:cost} \\
\textrm{s.t.} \quad  
    \bfx_{k+1} &= \bff(\bfx_k, \bfu_k, \bfw_k), \notag \\
    \Sigma_{k+1} &= A \left( \Sigma_k^{-1} + M (\bfx_{k+1}) \right)^{-1} A^\top + \Xi_k , \label{eq:Riccati} \\
    M(\bfx) &=  H(\bfx)^\top V(\bfx)^{-1} H(\bfx), \label{eq:M-def} 
\end{align}
where $\bar c_k: \bbS_{\succ 0}^{n_y \times n_y} \to \R$ for $k \in \{0, \dots, K\}$ is an objective function with respect to the covariance matrix, $M: \R^{n_x}  \to \R^{n_y \times n_y}$ is the so-called ``sensor information matrix"
. In $D$-optimal design to minimize \eqref{eq:cost-entropy}, we can set $\bar c_k = 0$ for $k \in \{0, \dots, K-1\}$ and $\bar c_K(\Sigma_K) = \log \det (\Sigma_K)$. In this paper, we keep setting a general objective function $\bar c_k$. 

We tackle the optimal control problem by a separate planning of an open-loop trajectory optimization (iCR) developed in \cite{koga2021active}, and a closed-loop feedback control by LQR. While iCR in \cite{koga2021active} has been developed for a robot motion model with an $SE(3)$ pose state represented by a matrix, 
a typical LQR approach is applicable to dynamics described with respect to a vector state. Therefore, we reformulate the Riccati update \eqref{eq:Riccati} of the covariance matrix by the dynamics of the vector state. Let $\bfsigma_k \in \R^{n_{\sigma}}$, where $n_{\sigma} := n_y (n_y+1) / 2$, be a vector representation of the covariance matrix $\Sigma_k $ defined by 
%
 $$ \bfsigma_k := \vech (\Sigma_k),$$ 
%
where $\vech (\cdot) : \bbS_{\succ 0}^{n_y \times n_y} \to \R^{n_y (n_y+1) / 2}$ is a half-vectorization operator applied to a symmetric matrix. 
By reformulating the Riccati update in \eqref{eq:Riccati} and
rewriting the cost function \eqref{eq:cost} with respect to the vector state $\bfsigma$, the optimization problem we consider can be recast as follows. 

\begin{problem*} 
Obtain a control policy $\bfu_k = \bfpi_k(\bfx_k, \bfsigma_k) $, where $\bfpi_k: \R^{n_x + n_{\sigma}} \to \R^m$ for $k = 0, \dots, K-1$ to solve the following stochastic optimal control problem:
\begin{align} 
    \min_{\bfpi_0, \dots, \bfpi_{K-1}} \E \left\{ \sum_{k=0}^{K-1} c_k(\bfsigma_k) + c_K (\bfsigma_K) \right\} \label{eq:cost-sigma}
\end{align}
where $c_k: \R^{n_{\sigma}} \to \R$, subject to 
\begin{align} \label{eq:AIA-motion}
    \bfx_{k+1} &= \bff(\bfx_k, \bfpi_k(\bfx_k, \bfsigma_k), \bfw_k), \\
    \bfsigma_{k+1} &= \bfg (\bfsigma_k, \bfx_{k+1}). \label{eq:AIA-sigma}
\end{align}
\end{problem*}

\section{Planning Method} \label{sec:planning} 

This section proposes a method to solve \eqref{eq:cost-sigma}--\eqref{eq:AIA-sigma} by designing an open-loop control sequence and closed-loop control policy separately. 

\subsection{Linearized system around iCR trajectory} 

iCR proposed by \cite{koga2021active} provides the solution to the deterministic case (i.e., $\bfw_k = 0$) of the optimal control stated above, for active exploration and mapping with the cost of log determinant of the covariance matrix. Extending iCR to minimizing a general cost, we obtain a nominal open-loop trajectory $\{\bar \bfx_{k+1}, \bar \bfu_{k}, \bar \bfsigma_{k+1} \}_{k = 0}^{K-1}$ which is a solution to 
\begin{align*}
    [\bar \bfu_1, \dots, \bar \bfu_{K-1}] &= \argmin_{\bar \bfu_1, \dots, \bar \bfu_{K-1}} \left( \sum_{k=0}^{K-1} c_k(\bar \bfsigma_k) + c_K (\bar \bfsigma_K) \right), \\
\textrm{s.t.} \quad    \bar \bfx_{k+1} &= \bff(\bar \bfx_k, \bar \bfu_k,0  ), \\
    \bar \bfsigma_{k+1} &= \bfg (\bar \bfsigma_k, \bar \bfx_{k+1}).  
\end{align*}
The solution is obtained via gradient-descent for a multi-step control sequence $\bfU = [\bar \bfu_1^\top, \dots, \bar \bfu_{K-1}^\top]^\top$, via iterative update of the control sequence by $\bfU \leftarrow \bfU - \bfalpha \fr{\pa J^{\bfU}}{\pa \bfU}$ with a step size $\bfalpha $ and the cost $J^{\bfU} = \sum_{k=0}^{K-1} c_k(\bfsigma_k) + c_K (\bfsigma_K)$, where the gradient $\fr{\pa J^{\bfU}}{\pa \bfU}$ is computed analytically. 
Around the nominal open-loop trajectory $\{\bar \bfx_{k+1}, \bar \bfu_{k}, \bar \bfsigma_{k+1} \}_{k = 0}^{K-1}$ and the mean of the noise $\bfw_k = 0$, the nonlinear stochastic dynamics \eqref{eq:AIA-motion}, \eqref{eq:AIA-sigma} can be described by a linear time-varying system as a first-order approximation through Taylor expansion (see \cite{zheng2021belief}). 
Let us define the error variables:
\begin{align*}
    \tilde \bfx_{k} &= \bfx_{k} - \bar \bfx_{k}, \quad  \tilde \bfsigma_k = \bfsigma_k - \bar \bfsigma_k, \\
    \tilde \bfpi_k(\tilde \bfx_k, \tilde \bfsigma_k) &= \bfpi_k(\bfx_k, \bfsigma_k) - \bar \bfu_{k} . 
\end{align*}
Then, linearizing the dynamics \eqref{eq:AIA-motion}, \eqref{eq:AIA-sigma} around the nominal trajectory $\{\bar \bfx_{k+1}, \bar \bfsigma_{k+1} \}_{k = 0}^{K-1}$, the dynamics for the error variables is given as 
:
\begin{align*}
       \tilde \bfx_{k+1} &= E_k \tilde \bfx_k + B_k \tilde \bfpi_k(\tilde \bfx_k, \tilde \bfsigma_k) + D_k \bfw_k, \\
    \tilde \bfsigma_{k+1} 
    &= F_k \tilde \bfsigma_k + G_k (E_k \tilde \bfx_k + B_k \tilde \bfu_k + D_k \bfw_k), 
\end{align*}
where 
\begin{align}
  \small  E_k &= \fr{\pa \bff}{\pa \bfx} \bigg |_{(\bar \bfx_k, \bar \bfu_k, 0)}, B_k = \fr{\pa \bff}{\pa \bfu}\bigg |_{(\bar \bfx_k, \bar \bfu_k, 0)},
  D_k = \fr{\pa \bff}{\pa \bfw}\bigg |_{(\bar \bfx_k, \bar \bfu_k, 0)}, \label{eq:Jacob-robot_2} \\
    F_k &= \fr{\pa \bfg}{\pa \bfsigma} \bigg |_{(\bar \bfx_k, \bar \bfu_k, 0)},\quad G_k =  \fr{\pa \bfg}{\pa \bfx} \bigg |_{(\bar \bfx_k, \bar \bfu_k, 0)}. \label{eq:Jacob-Riccati} 
\end{align}
We aim to minimize the cost \eqref{eq:cost-sigma} subject to the linearized stochastic dynamics by designing a control policy $\tilde \bfpi_k(\bfx, \bfsigma)$. Applying the Taylor expansion to \eqref{eq:cost-sigma} around the nominal trajectory and approximating by second-order accuracy yields
\begin{align}
    c_k(\bfsigma_k) & \approx  \tilde \bfsigma_k^\top C_k \tilde \bfsigma_k + \bar \bfb_k^\top  \tilde \bfsigma_k + c_k(\bar \bfsigma_k), \\
    C_k &:= \fr{\pa^2 c_k}{\pa \bfsigma^2} (\bar \bfsigma_k), \quad \bar \bfb_k^\top := \fr{\pa c_k}{\pa \bfsigma}(\bar \bfsigma_k). \label{eq:grad-cost} 
\end{align}
At the same time, to validate the linearization, the trajectories should stay around the nominal trajectories, namely, the error variables should stay around zero. Pursuing these two objectives, and introducing the new state variable $\bfs_k \in \R^{n_x + n_{\sigma}}$ defined by
 $$   \bfs_k = [\tilde \bfx_k^\top, \tilde \bfsigma_k^\top]^\top ,$$ 
the dynamics and the cost function can be written with respect to $\bfs_k$ as 
\begin{align}
    \bfs_{k+1} &= {\mathcal A}_k \bfs_k + {\mathcal B}_k \tilde \bfpi_k(\bfs_k) + {\mathcal D}_k \bfw_k, \label{eq:motion-LQR} \\
    \tilde J &=  \E \left\{ \sum_{k=0}^{K-1} ( \bfs_k^\top \calQ_k \bfs_k + \bfb_k^\top \bfs_k + \tilde \bfu_k R_k \tilde \bfu_k ) \right\} \notag\\
     & \quad +\E \left\{ \bfs_K^\top \calQ_K \bfs_K + \bfb_K^\top \bfs_{K}  \right\} , \label{eq:cost-LQR}
\end{align}
where $n_s = n_x + n_{\sigma}$,  $\calA_k \in \R^{ n_s \times n_s}$, $\calB_k  \in \R^{n_s \times m}$, $\calD_k \in \R^{n_s \times n_x}$, $\calQ \in \R^{n_s\times n_s}$, $\bfb \in \R^{n_s}$ are defined by 
\begin{align} \label{eq:def-cal} 
\calA_k &= 
\left[ \hspace{-2mm} 
\begin{array}{cc}
    E_k &  \hspace{-2mm} 0  \\
    G_k E_k &\hspace{-2mm} F_k
\end{array}\hspace{-2mm} 
\right] 
,
\calB_k = 
\left[ \hspace{-2mm} 
\begin{array}{c}
    B_k  \\
    G_k B_k
\end{array}\hspace{-2mm} 
\right] ,   
\calD_k = 
\left[ \hspace{-2mm} 
\begin{array}{c}
    D_k \\
    G_k D_k  
\end{array} \hspace{-2mm} 
\right] , \\
\calQ_k &= \left[ 
\begin{array}{cc}
    Q^{(1)}_k & 0 \\
    0 & Q^{(2)}_k + C_k
\end{array}
\right], \quad 
\bfb_k =
\left[ 
\begin{array}{c}
\bf0 \\
\bar \bfb_k
\end{array}
\right], \label{eq:calQ-def} 
\end{align}
where $Q_k^{(1)} \in \R^{n_x \times n_x}, Q_k^{(2)} \in \R^{n_{\sigma} \times n_{\sigma}}, $ and $R_k \in \R^{m \times m}$ are weight matrices to be determined by the user. 

\subsection{LQR-based closed-loop feedback control} 

We derive a closed-loop feedback control policy to minimize \eqref{eq:cost-LQR} subject to \eqref{eq:motion-LQR}. Note that, unlike the standard LQR, the cost function \eqref{eq:cost-LQR} includes 
a linear term $\bfb_k^\top \bfs_k$, in both stage and terminal costs. Even then, one can show that the optimal control policy is a linear feedback control with a time-varying constant term, as stated below.

\begin{proposition}
The closed-loop control
\begin{align}
     \tilde \bfpi_k(\bfs_k) &= L^*_k \bfs_k + \eps^*_k,  \\
     L^*_k &= - ( R + {\mathcal B}_k^\top P_{k+1} {\mathcal B}_k )^{-1} {\mathcal B}_k^\top P_{k+1} {\mathcal A}_k, \label{eq:Kgain} \\ 
     \eps^*_k  &= - \fr{1}{2} ( R + {\mathcal B}_k^\top P_{k+1} {\mathcal B}_k )^{-1}   {\mathcal B}_k^\top \bfd_{k+1},  \label{eq:kappagain} 
\end{align}
where $P_k \in \R^{n_s \times n_s}$, $\bfd_k \in \R^{n_s}$, and $\delta_k \in \R$ at $k = K$ are given by 
\begin{align}
    P_K = \calQ_K, \quad \bfd_K = \bfb_K, \quad \delta_K = 0, 
\end{align}
and recursively updated from $k+1$ to $k$ for $k = K-1, K-2, \dots, 0$ as follows:%
\begin{align}
    P_k &= \calQ_k + {\mathcal A}_k^\top P_{k+1} {\mathcal A}_k \notag\\
     & \hspace{4mm} - {\mathcal A}_k^\top P_{k+1}  {\mathcal B}_k ( R + {\mathcal B}_k^\top P_{k+1} {\mathcal B}_k )^{-1} {\mathcal B}_k^\top P_{k+1} {\mathcal A}_k, \label{eq:Q-update}\\
    \bfd_k &= \bfb_k + \calA_k^\top \bfd_{k+1} \notag\\
    & \hspace{4mm}-  \calA_k^\top P_{k+1} \calB_k  (R + \calB_k^\top P_{k+1} \calB_k)^{-1} \calB_k^\top \bfd_{k+1}, \label{eq:c-update} \\
    \delta_k &= \delta_{k+1}   +  \tr({\mathcal D}_k^\top P_{k+1}  {\mathcal D}_k W_k) \notag\\
    & \hspace{4mm}- \fr{1}{4} \bfd_{k+1}^\top {\mathcal B}_k  ( R + {\mathcal B}_k^\top P_{k+1} {\mathcal B}_k )^{-1}   {\mathcal B}_k^\top \bfd_{k+1} \label{eq:delta-update}
\end{align}
minimizes the cost function \eqref{eq:cost-LQR} subject to the system dynamics \eqref{eq:motion-LQR} with optimal cost:
\begin{align*}
    \min_{\tilde \bfu_0, \dots, \tilde \bfu_{K-1}} \tilde J = V_0(\bfs) =  \bfs^\top P_0 \bfs + \bfd_0^\top \bfs + \delta_0, 
\end{align*}
for a given initial condition $\bfs_0 = \bfs$. 
\end{proposition}

The proof is done using a dynamic programming method, and is omitted in this paper due to space constraints. 


\section{Application to Active SLAM} \label{sec:SLAM} 

We apply the proposed planning method to an active SLAM problem, where a set of landmarks in the environment is regarded as the target state. The task is to estimate the landmark positions and the pose of a sensing robot, and to plan motion that reduces the uncertainty in these estimates.

\subsection{Differential-drive motion model and its linearization} 
 Let $\bfx = [\bfp^\top, \theta]^\top \in \R^3$ be the state of a ground robot, where $\bfp \in \R^2$ is the robot position and $\theta \in [-\pi, \pi)$ is the robot's heading angle. Let $\bfu \in [v, \omega] \in \R^2$ be the robot's control input, where $v$ is the linear velocity and $\omega$ is the angular velocity. We model the robot dynamics $\bff: \R^3 \times \R^2 \times \R^3 \to \R^3$ in \eqref{eq:AIA-motion} using a differential-drive kinematic model with time discretization $\tau$:
 \begin{align} 
 \bff(\bfx_k, \bfu_k, \bfw_k) &= \bfx_k + \tau \left[ 
 \begin{array}{c} 
 v_k \sinc \left( a_k \right) \cos \left( \theta_k + a_k \right) \\
  v_k \sinc \left( a_k \right) \sin \left( \theta_k +  a_k \right) \\
  \omega_k 
  \end{array} 
  \right]+ \bfw_k, \label{eq:SE2-motion}
\end{align} 
where $a_k := \fr{\omega_k \tau}{2}$. 
%
%

\subsection{Limited field-of-view sensing model}
The landmarks are modeled as points $\bfy^{(j)} \in \R^2$ in the environment, for all $j \in \{1, \dots, n_l\}$, where $n_l$ is a number of landmark. Noting that the landmark location is static, the matrices in the target dynamics given in \eqref{eq:AIA1} are set as:
\begin{align} \label{eq:mapping-static}
    A = I_{2n_l \times 2 n_l}, \quad \Xi_k = 0_{2 n_l \times 2 n_l}.  
\end{align}
For a given robot state $\bfx \in \R^3$ and the position $\bfy^{(j)} \in \R^2$ of the $j$-th landmark, we consider the robot body-frame coordinates of $\bfy^{(j)}$:
\begin{align}
   \bfq \left(\bfx, \bfy^{(j)} \right) =  R^\top (\theta) (\bfy^{(j)} - \bfp), \label{eq:hbar-def}
\end{align}
where $R: \R \to SO(2) \subset \R^{2 \times 2}$ is a 2-D rotation matrix of the robot pose. 
Here, we define the set of indices of the landmarks within the field of view $\calF \subset \R^2 $ of the robot, as follows:
\begin{align}
    \calI_{k, {\calF}} = \left\{ \forall j \in \{1, \dots, n_l\} \bigg|   \bfq \left(\bfx_k, \bfy^{(j)} \right)  \in \calF \right\}
\end{align}
We suppose to have both range and bearing measurements, which capture the relative landmark positions in robot frame as follows:
\begin{align}
    \bfz_k &= \left[ 
    \begin{array}{c} 
    \{ \bar \bfz(\bfx_k, \bfy^{(j)})\}_{j \in \calI_{k, \calF}}
    \end{array} 
    \right] \in \R^{2 |\calI_{k, \calF}|}, \label{eq:sensor-slam} \\
  \bar  \bfz(\bfx, \bfy^{(j)}) &=   \bfq (\bfx, \bfy^{(j)} )+ \Gamma^{1/2}(\bfx, \bfy^{(j)}) \bfv, \label{eq:slam-sensing} 
\end{align}
where $\Gamma: \R^3 \times \R^2 \to \bbS_{\succ 0}^{2 \times 2}$ is the sensor noise covariance and $\bfv \sim \N (0, I_2)$. 
Notice that the sensing model \eqref{eq:slam-sensing} is nonlinear with respect to the target state, while up to the previous sections we have considered a linear sensing model. In the application to active SLAM in this section, we deal with the nonlinear sensing model by employing the linearization around some estimate of the target state, and utilize it for both planning and SLAM.

\subsection{Differentiable field of view}

One special characteristic of the sensing model \eqref{eq:sensor-slam} with a limited field of view (FoV) is that the measurement dimension is dependent on the robot pose state at time $k$. This is caused by a binary (observable or unobservable) sensing within the FoV, which makes the sensing model non-differentiable with respect to the state and hence the control input. To deal with the challenge, we use a differentiable FoV proposed in \cite{koga2021active}, where the measurement is supposed to be obtained for all landmark states, while the noise covariance is supposed to become approximately infinity outside the FoV. Namely, the measurement function $\bfh: \R^3 \times \R^{2 n_l} \to \R^{2 n_l}$ and the noise covariance matrix are formalized as:
\begin{align}
    \bfh(\bfx, \bfy) &= [ \bfq(\bfx, \bfy^{(1)})^\top, \dots, \bfq(\bfx, \bfy^{(n_l)})^\top]^\top , \label{eq:h-def-fov} \\
    V(\bfx) &=  \diag(\bar V(\bfx, \hat \bfy^{(1)}), \dots, \bar V(\bfx, \hat \bfy^{(n_l)})) \in \R^{2 n_l \times 2 n_l}, \notag
\end{align}
where $\hat \bfy^{(j)} \in \R^{2}$ is an initial estimate of the $j$-th landmark state, $\bar V(\bfx, \bfy^{(j)}): \R^3 \times \R^2 \to \R^{2 \times 2}$ is a noise covariance of $j$-th landmark with differentiable FoV, given by 
\begin{align}
       \bar V(\bfx, \bfy^{(j)})  &= \left( 1 -  \Phi( d(\bfq(\bfx, \bfy^{(j)}), {\mathcal F})) \right)^{-1} \Gamma(\bfx, \bfy^{(j)}),  
          \notag 
          \end{align} 
where $\Phi : \R \to [0, \;1]$ is the Gaussian CDF defined by $\Phi(x) = \frac{1}{2} \left[ 1 + \textrm{erf} \left(\frac{x}{\sqrt{2} \kappa} - 2  \right) \right]$
and $d(\bfq, \calF)$ is a signed distance function defined below.

\begin{definition}
The \emph{signed distance function} $d : \R^2 \to \R $ associated with a set $\calF \subset \bbR^2$ is:
\begin{align}
     d(\bfq, {\mathcal F})  = 
     \begin{cases}
     - \min_{\bfq^* \in \pa {\mathcal F}}  \| \bfq - \bfq^*\|, & \textrm{if} \quad \bfq \in {\mathcal F}, \\
     \phantom{-} \min_{\bfq^* \in \pa {\mathcal F}} \|\bfq - \bfq^*\|, & \textrm{if} \quad \bfq \notin {\mathcal F},  
     \end{cases}
\end{align}
where $\pa {\mathcal F}$ is the boundary of ${\mathcal F}$.
\end{definition}

\begin{figure}[t]
 \vspace{1.5mm}
 \centering
  \includegraphics[width=0.8\linewidth]{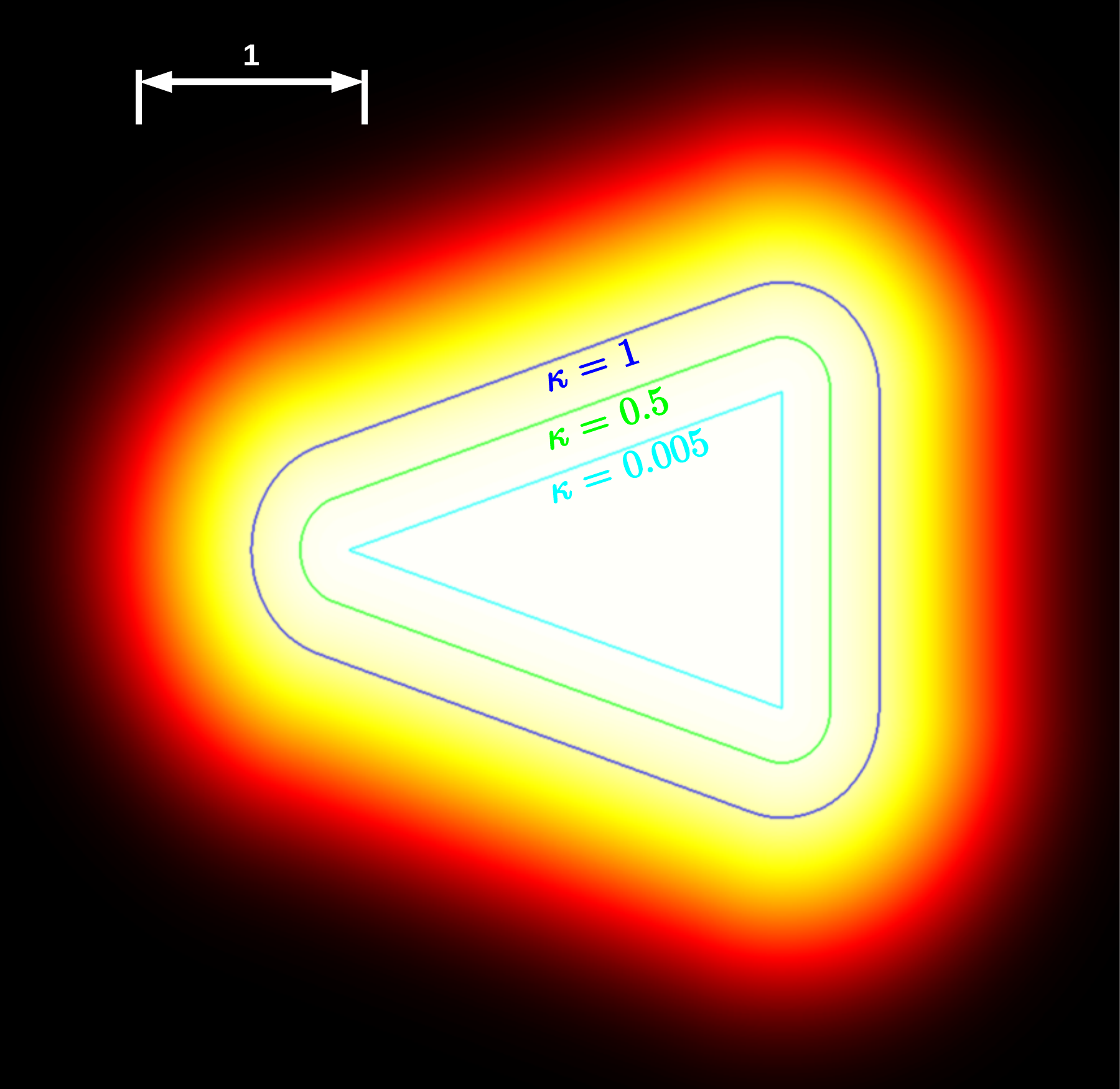}%
  \hfill%
  \caption{Visualization of differentiable FoV. The parameter $\kappa>0$ is tuned to obtain a desired smoothness in the FoV. }
  \label{fig:fov}
\end{figure}

A 2-D plot of $1 - \Phi( d(\bfq(\bfx, \bfy^{(j)}), {\mathcal F})$, which is the amplification factor in the differentiable FoV, is shown in Fig. \ref{fig:fov}. We observe the smooth transition from observable to unobservable space, which enables the gradient computation of the sensor information matrix needed in iCR.  In accordance with the linear sensing model in \eqref{eq:AIA1}, we take a linearization of \eqref{eq:h-def-fov}. Let $H: \R^3 \to \R^{2 n_l \times 2 n_l}$ be the measurement matrix defined by $H(\bfx) = \fr{\pa \bfh}{\pa \bfy}(\bfx, \hat \bfy)$. Then, with \eqref{eq:h-def-fov}, one can show $H$ is a block diagonal matrix, given by $H(\bfx) =  \diag(\bar H(\bfx_k, \hat \bfy^{(1)}), \dots, \bar H(\bfx_k, \hat \bfy^{(n_l)})) $, where $\bar H(\bfx_k, \bfy^{(j)}) : \R^3 \times \R^2 \to \R^{2 \times 2}$ is the measurement matrix for $j$-th landmark at time $k$, given as 
\begin{align}  \label{eq:Hj-def} 
\bar H(\bfx_k, \bfy^{(j)}) := \fr{\pa \bfq}{\pa \bfy^{(j)}}(\bfx_k, \bfy^{(j)}) = R^\top (\theta_k),
\end{align}
where we utilized \eqref{eq:hbar-def}. 
Since the product of the block diagonal matrix with same size is also a block diagonal matrix of each product, the sensor information matrix defined by \eqref{eq:M-def} is also a block diagonal matrix, given by $M(\bfx) =  \diag ( \bar M(\bfx, \hat \bfy^{(1)}), \dots, \bar M(\bfx, \hat \bfy^{(n_l)}) )$, where 
\begin{align*}
    \bar M(\bfx, \bfy^{(j)}) 
    &= \left( 1 -  \Phi( d( \bfq^{(j)}, {\mathcal F})) \right) R(\theta) \Gamma^{-1} R^\top (\theta)
\end{align*}
Since the matrix $M$ and the initial covariance matrix are block diagonal matrices with same size, and the fact that the landmark is static \eqref{eq:mapping-static}, Riccati update \eqref{eq:Riccati} leads to that the covariance matrix $\Sigma_k$ is also a block diagonal matrix with  the same size for all $k$, i.e., $\Sigma_k = \diag( \Sigma_k^{(1)}, \Sigma_k^{(2)}, \cdots, \Sigma_k^{(n_l)} )$, 
where each element satisfies 
\begin{align}
 \Sigma^{(j)}_{k+1} = \left ( \left (\Sigma^{(j)}_{k} \right)^{-1} + \bar M \left(\bfx_{k+1}, \hat \bfy^{(j)}_0 \right) \right)^{-1} \label{eq:Riccati-diag}
\end{align}

\subsection{iCR for active SLAM with trace minimization}
Following \cite{koga2021active}, 
in the derivation of iCR, the robot pose dynamics is described by $SE(2)$ pose kinematics, in which the robot state is defined in $SE(2)
\subset \R^{3 \times 3}$ 
. The $SE(2)$ pose kinematics is equivalent to the differential-drive motion model in \eqref{eq:SE2-motion}, through the exponential map $T = \exp(\hat \bfx)$ and the logarithm map $\bfx = \textrm{log}\left( T \right)^{\vee}$ \cite{barfoot2017state}. While in \cite{koga2021active} the cost is set to a log determinant of the covariance matrix at final time, in this paper we consider the cost to be a trace of the covariance matrix, which eases the computation of the gradient \eqref{eq:grad-cost} of the cost with respect to the vectorized covariance. Describing the robot dynamics by $SE(2)$ pose kinematics, and taking into account the differentiable field of view formulation presented above, iCR for active SLAM problem is developed, thereby the nominal trajectory $\{\bar T_{0:K}, \bar \Sigma_{0:K}, \bar \bfu_{0,K-1}\}$ for a deterministic $SE(2)$ pose kinematics is obtained.

\subsection{LQR gain derivation}

As presented in Secion \ref{sec:planning}, for deriving the LQR gains, provided the nominal iCR trajectory, Jacobian matrices \eqref{eq:Jacob-robot_2} \eqref{eq:Jacob-Riccati} must be computed. As stated above, iCR trajectory $\{\bar T_{0:K}, \bar \Sigma_{0:K}, \bar \bfu_{0,K-1}\}$ is obtained for a $SE(2)$ pose state $\bar T \in SE(2) \subset \R^{3 \times 3}$ and the covariance matrix $\bar \Sigma \in \bbS_{\succ 0}^{2 n_l \times 2n_l}$, which need to be converted to vectorized states. The robot vector state $\bfx \in \R^3$ is obtained by $\bfx = \textrm{log}\left( T \right)^{\vee}$. Regarding the covariance vector state $\bfsigma$, in Section \ref{sec:problem}, we define by the half-vectorization operator $\vech(\cdot)$ as a general case. However, in active SLAM we consider in this section, owing to the independency among covariances of each landmark as shown in \eqref{eq:Riccati-diag} and its symmetric property, we can define a covariance vector state with a lower dimension than the half-vectorization operator as follows. Let $\bfsigma^{(j)} \in \R^3$ be defined by $
    \bfsigma^{(j)} =
    \left[ 
    \bfe_1^\top \Sigma^{(j)} \bfe_1, 
    \bfe_1^\top \Sigma^{(j)} \bfe_2, 
    \bfe_2^\top \Sigma^{(j)} \bfe_2 
    \right]^\top $, where $\Sigma^{(j)} \in  \bbS_{\succ 0}^{2 \times 2}$ is a $j$-th block diagonal matrix in $\Sigma \in \bbS_{\succ 0}^{2 n_l \times 2n_l}$. Then, we define the covariance vector state by $ \bfsigma := [\bfsigma^{(1)}, \bfsigma^{(2)}, \dots, \bfsigma^{(n_l)}] \in \R^{3 n_l}$, which includes all the variables in the covariance matrix obeying the Riccati update. Hereafter, this conversion is denoted as $\bfsigma = \textrm{vecbl} (\Sigma)$.

    Through this conversion from the covariance matrix to the covariance vector state, we can derive the Jacobian matrix analytically. The computation steps of obtaining the LQR gain is shown in Algorithm 1, which includes "Jacob-robot" as a Jacobian matrices of the robot dynamics \eqref{eq:SE2-motion}, and "Jacob-Riccati" as a Jacobian matrices of the Riccati update provided in Appendix. "Grad-cost" computes the gradient of the cost function with respect to the vectorized covariance state as given in \eqref{eq:grad-cost}. Considering the cost of minimizing a trace of the covariance, the variables \eqref{eq:grad-cost} in the cost function are obtained as $C_k = 0_{3 n_l \times 3 n_l}$ and $\bar \bfb_k = [\bfeta^\top, \bfeta^\top, \dots, \bfeta^\top]^\top \in \R^{3 n_l}$ where $\bfeta = [1, 0, 1]^\top \in \R^3$, since each $\bfeta$ in $\bar \bfb_k$ corresponds to the diagonal element of $\Sigma^{(j)}$ with respect to the vectorized covariance $\bfsigma^{(j)}$.

\RestyleAlgo{boxruled}
\begin{algorithm}
	\caption{ LQR gain}
	\KwData{iCR trajectory $\bar T_{1:K}, \bar \bfu_{1:K-1},  \bar \Sigma_{1:K}$, weight matrices $Q_k^{(1)}, Q_k^{(2)}, R_k$} 
	$C, \bar \bfb \gets $ Grad-cost($\bar \bfsigma_K$) in \eqref{eq:grad-cost} \\
	$[P, \bfd, \delta] \gets [\calQ_K, \bfb_K, 0] $ from \eqref{eq:calQ-def} \\
	Solve LQR gain backward in time\\
	\For{$k\gets K-1$ \KwTo $0$}{
	$[\bar \bfx_k, \bar \bfx_{k+1}, \bar \bfsigma_k] \gets [\textrm{log}\left(\bar T_k \right)^{\vee}, \textrm{log}\left(\bar T_{k+1} \right)^{\vee}, \textrm{vecbl}\left( \bar \Sigma_k \right)]$ 
	\\
	$E$, $B$, $D$ $\gets$ Jacob-robot($\bar \bfx_k$, $\bar \bfu_k$)\\
	$F$, $G \gets $ Jacob-Riccati($\bar \bfsigma_k$, $\bar \bfx_{k+1}$) \\
	$\calA, \calB, \calD \gets $ \eqref{eq:def-cal} using $[E, B, D, F, G]$\\
	$C, \bar \bfb \gets $ Grad-cost($\bar \bfsigma_k$) in \eqref{eq:grad-cost} \\
	$\calQ, \bfb \gets $ \eqref{eq:calQ-def} using $[Q_k^{(1)}, Q_k^{(2)}, C, \bar \bfb]$ \\
	$[L^*_k, \eps^*_k] \gets $ \eqref{eq:Kgain}, \eqref{eq:kappagain} using $[P, \bfd, \calA, \calB, R] $\\
	$\bfd \gets $ \eqref{eq:c-update} using $[\bfd, \bfb, P, \calA, \calB, R]$ \\
	$\delta \gets $ \eqref{eq:delta-update} using $[\delta, P, \bfd, \calB, R, \calD, W_k]$ \\
	$P \gets $ \eqref{eq:Q-update} using $[P, \calQ, \calA, \calB, R]$ \\
    }
  \textbf{Output:} $L^*_{0:K-1}, \eps^*_{0:K-1}$ 
\end{algorithm}

\begin{figure}[t]
 \vspace{1.5mm}
 \centering
  \includegraphics[width=\linewidth]{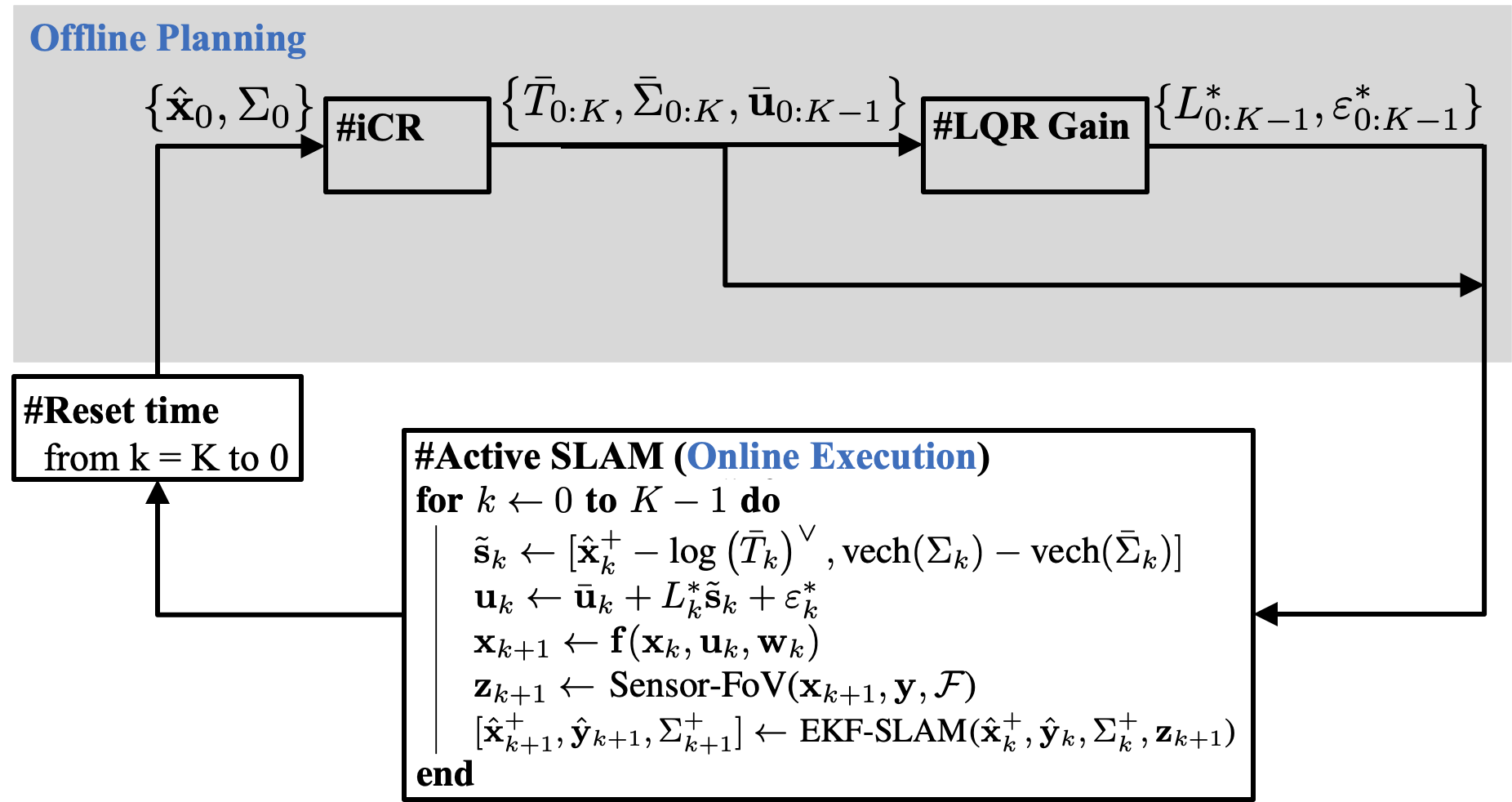}%
  \hfill%
  \caption{Diagram of active SLAM via iCR-LQR.}
  \label{fig:icr-lqr}
\end{figure}

\subsection{EKF-SLAM} \label{subsec:EKF}

We construct the estimator of both the landmark position and the robot position by Extended Kalman Filter (EKF). 
The probabilistic landmark position and the robot position are set as a Gaussian distribution: $\left[ 
    \begin{array}{c}
    \bfx_k \\ \bfy
    \end{array} 
    \right] 
    | \bfz_{1:k} \sim  \N \left(
    \left[
    \begin{array}{c}
    \hat \bfx_k^+ \\ \hat \bfy_k
    \end{array} 
    \right] , 
    \Sigma_k^+, 
    \right) $ for posteriori estimate, and $
  \left[ 
    \begin{array}{c}
    \bfx_{k+1} \\ \bfy
    \end{array} 
    \right] 
    | \bfz_{1:k} \sim  \N \left(
    \left[
    \begin{array}{c}
    \hat \bfx_{k+1}^- \\ \hat \bfy_k
    \end{array} 
    \right] , 
    \Sigma_{k+1}^-, 
    \right) 
$ for a priori estimate. Once the measurement is obtained as given in \eqref{eq:sensor-slam}, which is a measured landmark relative position in robot-body frame within FoV, we reconstruct the measured state as $\bfz_k = [\bfz_k^{(1)}, \bfz_k^{(2)}, \dots, \bfz_k^{(n_l)}] \in \R^{2 n_l}$, where 
\begin{align}
    \bfz_k^{(j)} = \begin{cases} 
    \bar \bfz(\bfx_k, \bfy^{(j)}), \quad \textrm{if } \quad  j \in \calI_{k, \calF} \\
     \bfq(\hat \bfx_k^-, \hat \bfy_{k-1}^{(j)}), \quad \textrm{otherwise} 
    \end{cases} \label{eq:sensor-recon} 
\end{align}
Using the reconstructed sensor state \eqref{eq:sensor-recon}, we implement EKF for SLAM by updating the mean and covariance of both priori and posteriori estimates. 
Note that, since the innovation term in EKF is $\bfz_k - \bfh(\hat \bfx_k^-, \hat \bfy_{k-1})$, where $\bfh$ is given by \eqref{eq:h-def-fov}, the reconstructed sensor state \eqref{eq:sensor-recon} makes the innovation term zero in $j$-th landmark estimate for all $j \notin \calI_{k, \calF}$. Namely, all the landmark estimate outside FoV does not have an update through applying \eqref{eq:sensor-recon} to EKF-SLAM.


A diagram depicting the structure of the entire proposed algorithm is shown in Fig. \ref{fig:icr-lqr}. 

\subsection{Evaluation} 

\begin{figure}[t]
    \centering
    \begin{subfigure}[t]{0.48\linewidth}
    {%
    \setlength{\fboxsep}{0pt}%
    \setlength{\fboxrule}{0.01pt}%
    \fbox{\includegraphics[width=\linewidth]{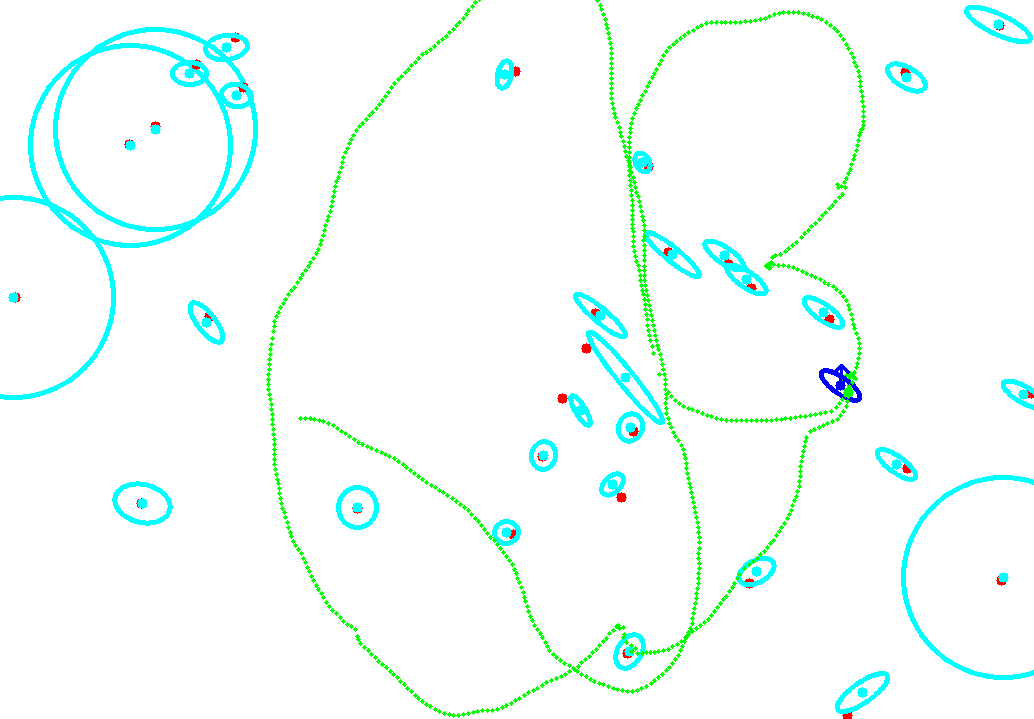}}%
    }%
    \caption{Closed-loop iCR + LQR}
    \end{subfigure}%
    \hfill%
    \begin{subfigure}[t]{0.48\linewidth}
    {%
    \setlength{\fboxsep}{0pt}%
    \setlength{\fboxrule}{0.01pt}%
    \fbox{\includegraphics[width=\linewidth]{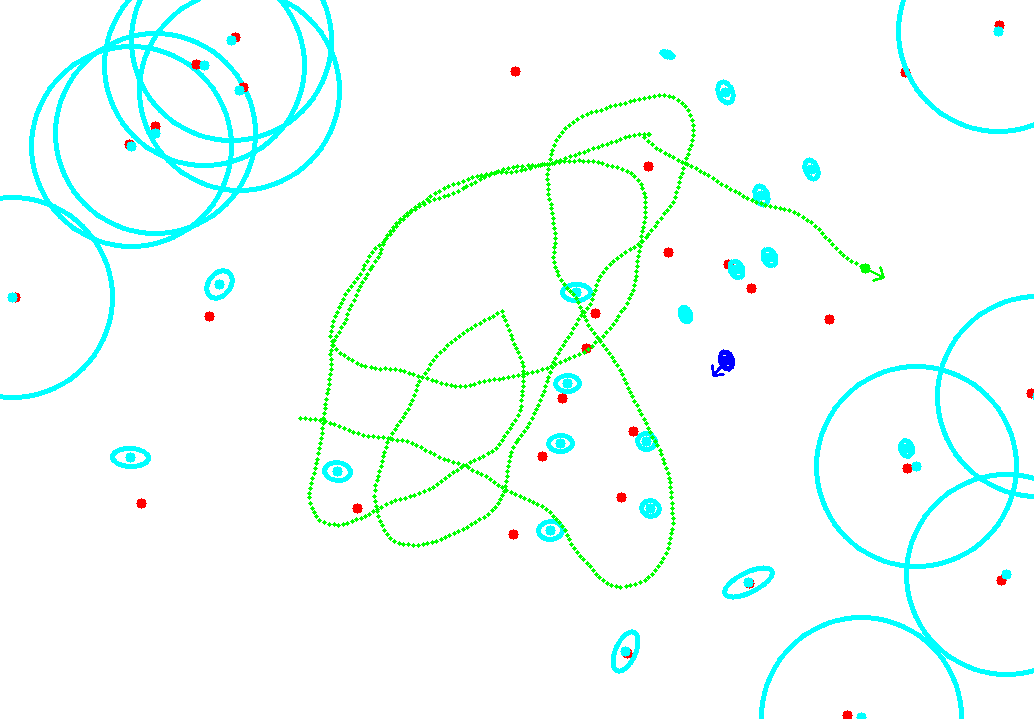}}%
    }%
    \caption{Open-loop iCR}
    \end{subfigure}
    \caption{Active SLAM via closed-loop and open-loop control policies. The green dotted line shows the robot ground truth trajectory. The blue dot shows robot pose, while the surrounding ellipse corresponds to the covariance of robot position. Red dots indicate the ground truth landmark positions. Cyan dots and ellipses indicate the mean and covariance of landmarks, respectively.}
    \label{fig:timelapse}
\end{figure}

\begin{figure*}[t]
    \centering
    \includegraphics[width=\linewidth]{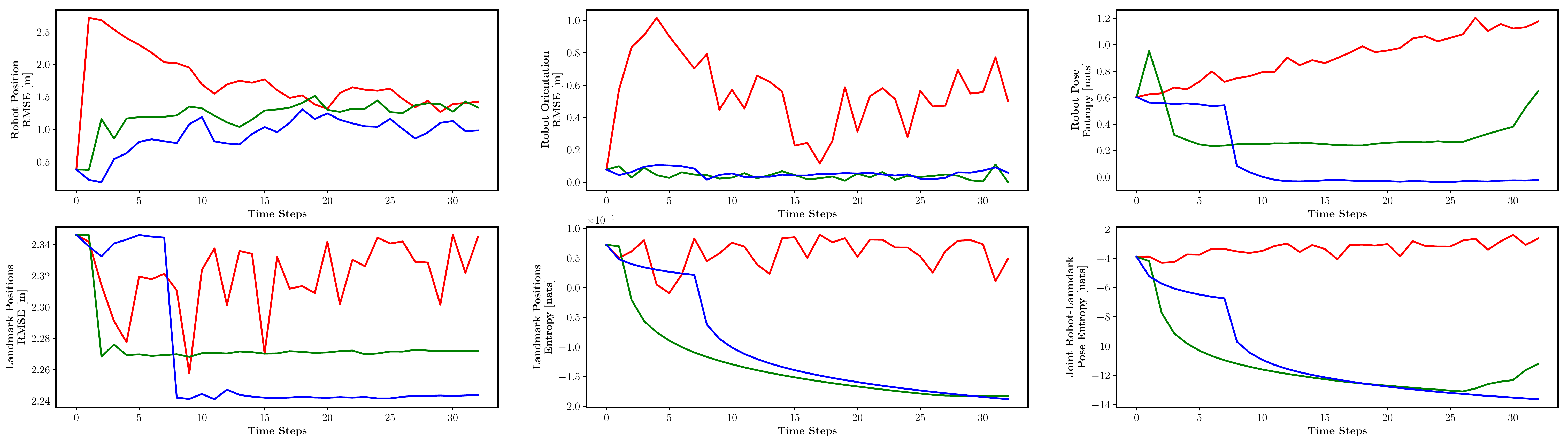}
    \caption{Simulation results for random policy (red), iCR (green), and iCR with closed-loop LQR control (blue). Top three figures illustrate the root mean square error (RMSE) of robot position and orientation, as well as robot pose entropy. The bottom figures show RMSE of landmark positions, average landmark position entropy, and the joint robot-landmark pose entropy. The results are averaged over $5$ random environments in which all three methods were executed.}
    \label{fig:sim_res}
\end{figure*}

We examine the performance of the proposed method in a simulated environment with dimensions $100 \text{[m]} \times 70 \text{[m]}$, where the landmarks are located following a uniform random distribution. The robot follows the $SE(2)$ motion model of \eqref{eq:SE2-motion} with $W_k = \diag \left(0.1, 0.1, 0.01\right)$, and its on-board sensor measures the relative position of visible landmarks in the robot frame. The field of view ${\mathcal F}$ is set as an isosceles triangle with height $20$ [m] and the angle between the two legs equal to $120^{\circ}$. The relative position measurements are corrupted by an additive Gaussian noise with zero mean and covariance $\Gamma = \diag \left(0.1, 0.1\right)$. The control $\bfu_k$ and measurements $\bfz_{k+1}$ are given to EKF-SLAM for state estimation, where we assume the noise covariances $W_k$ and $\Gamma$ are known to EKF-SLAM. We initialize the mean $\left[\hat{\bf{x}}^+_0 \hat{\bf{y}}^+_0\right]^\top$ with the ground truth position of the robot and landmarks, added by a Gaussian noise with variance $25$ [m$^2$], while the state covariance is initialized as $\Sigma^+_0 = 25 I_{n_s \times n_s}$. During each planning phase, we begin by computing the initial iCR control sequence $\bar{\bfu}_{0:K-1}$ for planning horizon $K = 5$, where the differentiable FoV is parameterized by $\kappa = 10$ and the gradient decent update is done for $10$ iterations with $\bfalpha = \diag \left(0.005, 0.0005\right)$. The obtained control sequence $\bar{\bfu}_{0:K-1}$ is regularized by LQR where $Q_k^{(1)} = \diag \left(10, 10, 1\right)$, $Q_k^{(2)} = I_{n_l \times n_l} \otimes \diag \left(1, 0.1, 1\right)$, and $R_k = 
\left[
\begin{array}{cc}
    20 & 5 \\
    5 & 10
\end{array}
\right]$. The closed-loop control $\bfu_k$ is applied to the robot for $K$ steps while the state mean and covariance are updated on each step. Fig.~\ref{fig:timelapse} shows examples of active SLAM using open-loop and closed-loop control policies. We observe that the LQR closed-loop control allows for larger exploration of the environment since the trajectory is constantly corrected by LQR, while for the case of open-loop policy, landmark entropy increases after execution of the control sequence, which encourages re-visiting the nearby landmarks and resulting in limited exploration.

Fig.~\ref{fig:sim_res} summarizes simulation results for a random policy, open-loop control obtained from iCR, and LQR closed-loop control over iCR output. Both open-loop and closed-loop policies outperform the random policy; however, the policy regulated by LQR shows more long-term stability and uncertainty reduction. This can be directly attributed to the cost function for LQR, where stability and landmark position uncertainty are explicitly factored in the model. 


\section{Conclusion} \label{sec:conclusion} 

This paper developed a method for continuous trajectory optimization for active information acquisition problems.
The problem is formalized as a stochastic optimal control problem to minimize an uncertainty measure over the target state. The novelty of the proposed method lies in (i) taking into account the process noise in robot dynamics, (ii) introducing a differentiable field of view for enabling the gradient computation, and (iii) planning an open-loop trajectory by iCR and a closed-loop control policy by LQR applied to a linearized system around iCR trajectory. We demonstrated the efficacy of the proposed method in a simulation of landmark-based active SLAM, aiming to map the landmarks and localize robot accurately.  

\bibliographystyle{IEEEtran}
\bibliography{BIB_ACC2022.bib}

\appendix

\subsection{Jacobian matrix of Riccati update}  
Let $\bfsigma^{(j)}_k \in \R^3$ and $\bfm \in \R^3$ be the vectors defined by 
\begin{align}
    \bfsigma^{(j)}_k &=
    \left[ 
    \begin{array}{c} 
    \bfe_1^\top \Sigma^{(j)}_k \bfe_1 \\
    \bfe_1^\top \Sigma^{(j)}_k \bfe_2  \\
    \bfe_2^\top \Sigma^{(j)}_k \bfe_2 
    \end{array}
    \right] ,
    \bfm^{(j)}(\bfx) =
    \left[ 
    \begin{array}{c} 
    \bfe_1^\top \bar M(\bfx, \hat \bfy^{(j)}_0) \bfe_1 \\
    \bfe_1^\top \bar M(\bfx, \hat \bfy^{(j)}_0)  \bfe_2  \\
    \bfe_2^\top \bar M(\bfx, \hat \bfy^{(j)}_0)  \bfe_2 
    \end{array}
    \right] \notag 
\end{align}
Let the entire vector states be defined by $ \bfsigma := [\bfsigma^{(1)}, \bfsigma^{(2)}, \dots, \bfsigma^{(n_l)}] \in \R^{3 n_l}$ and the update function $ \bfg (\bfsigma, \bfx) = \left[ 
    \begin{array}{ccc}
         \bar \bfg (\bfsigma^{(1)}, \bfm^{(1)}(\bfx)) &  \cdots &  \bar \bfg (\bfsigma^{(n_l)}, \bfm^{(n_l)}(\bfx))
    \end{array}
    \right] $. 
Then,  
the Jacobian matrices in \eqref{eq:Jacob-Riccati} are obtained explicitly as stated in the following proposition.  
\begin{proposition} \label{prop:jacob-riccati}
The Jacobian matrices \eqref{eq:Jacob-Riccati} of the vector dynamics of Riccati update can be obtained by 
\begingroup
\allowdisplaybreaks
\begin{align*}
     &\hspace{-2mm} F_k =  \fr{\pa \bfg}{\pa \bfsigma} = \diag(F_k^{(1)}, \dots, F_k^{(n_l)}), \\
     &\hspace{-2mm} F_k^{(j)} =   
    \left[ 
    \begin{array}{c} 
    \fr{\pa g_1}{\pa \bfsigma}(\bfsigma, \bfm) \\
    \fr{\pa g_2}{\pa \bfsigma}(\bfsigma, \bfm) \\
    \fr{\pa g_3}{\pa \bfsigma}(\bfsigma, \bfm)
    \end{array} 
    \right] \bigg |_{[\bfsigma,\bfm] =[ \bfsigma^{(j)}_k, \bfm^{(j)}(\bfx_{k+1})]}, \\
    &\hspace{-2mm} G_k = \fr{\pa \bfg}{\pa \bfx} 
    =  \left[ 
    \begin{array}{c}
        \fr{\pa \bar \bfg}{\pa \bfx} (\bfsigma^{(1)}_k, \bfm^{(1)}(\bfx_{k+1})) \\
          \vdots \\
         \fr{\pa \bar \bfg}{\pa \bfx} (\bfsigma^{(n_l)}_k, \bfm^{(n_l)}(\bfx_{k+1}))
    \end{array}
    \right] , \\
    & \hspace{-2mm}\fr{\pa \bar \bfg}{\pa \bfx} 
    =  \left[ 
    \begin{array}{c} 
    \fr{\pa g_1}{\pa \bfm}(\bfsigma, \bfm) \\
    \fr{\pa g_2}{\pa \bfm}(\bfsigma, \bfm) \\
    \fr{\pa g_3}{\pa \bfm}(\bfsigma, \bfm)
    \end{array} 
    \right]  
    \left[ 
    \begin{array}{c} 
    \fr{\pa }{\pa \bfx} \left( \bfe_1^\top  M^{(j)}(\bfx, \hat \bfy^{(j)}_0) \bfe_1 \right) \\
    \fr{\pa }{\pa \bfx} \left(\bfe_1^\top  M^{(j)}(\bfx, \hat \bfy^{(j)}_0)  \bfe_2 \right) \\
    \fr{\pa }{\pa \bfx} \left(\bfe_2^\top  M^{(j)}(\bfx, \hat \bfy^{(j)}_0)  \bfe_2 \right)
    \end{array}
    \right] 
\end{align*}
\endgroup
where $\fr{\pa g_i}{\pa \bfsigma}: \R^3 \times  \R^3 \to \R^3$ and $\fr{\pa g_i}{\pa \bfm}: \R^3 \times  \R^3 \to \R^3$ as functions of $\bfsigma = [\sigma_1, \sigma_2, \sigma_3] \in \R^3$ and $\bfm = [m_1, m_2, m_3] \in \R^3$ are given by 
\begin{align*}
\fr{\pa g_i}{\pa \bfsigma}(\bfsigma, \bfm) &= f(\bfsigma, \bfm)^{-1} \left( \bfr_i -   g_i(\bfsigma, \bfm) \fr{\pa f}{\pa \bfsigma} \right), \\
  \fr{\pa f}{\pa \bfsigma}(\bfsigma, \bfm) &= \left[  
  \begin{array}{c}
  m_1  + \sigma_3 (m_1 m_3 -  m_2^2) \\
  2 m_2 - 2 \sigma_2 (m_1 m_3 -  m_2^2)  \\
   m_3 + \sigma_1 (m_1 m_3 -  m_2^2) 
  \end{array} 
  \right]^\top , \end{align*}
\begin{align*}   
\bfr_1 &= \left[ 
\begin{array}{ccc}
1 + \sigma_3 m_3 &  - 2 \sigma_2 m_3 &   \sigma_1 m_3
\end{array} 
\right] , \\
 \bfr_2 &= \left[
 \begin{array}{ccc}
- \sigma_3 m_2  & 1 + 2  \sigma_2 m_2  &  - \sigma_1 m_2
\end{array} 
\right], \\
\bfr_3 &=  \left[ 
\begin{array}{ccc}
\sigma_3 m_1 & - 2 \sigma_2 m_1 & 1 + \sigma_1 m_1 
\end{array} 
\right] , 
\end{align*}
\begin{align*}
\fr{\pa g_i}{\pa \bfm}(\bfsigma, \bfm) &=f(\bfsigma, \bfm)^{-1} \left((\sigma_1 \sigma_3 - \sigma_2^2) \widetilde \bfr_i^\top -   g_i(\bfsigma, \bfm) \fr{\pa f}{\pa \bfm} \right)
\\
  \fr{\pa f}{\pa \bfm}(\bfsigma, \bfm) &= \left[  
  \begin{array}{c}
  \sigma_1 + ( \sigma_1 \sigma_3 - \sigma_2^2) m_3  \\
  2 \sigma_2 - 2 (\sigma_1 \sigma_3 -  \sigma_2^2 )  m_2  \\
   \sigma_3 + ( \sigma_1 \sigma_3 - \sigma_2^2) m_1
  \end{array} 
  \right]^\top , 
\end{align*}
$ \widetilde \bfr_1 = \bfe_3, \widetilde\bfr_2 = - \bfe_2, \widetilde \bfr_3 = \bfe_1$, and $\fr{\pa \bfe_i^\top M^{(j)} \bfe_l}{\pa \bfx} : \R^3 \times \R^2 \to \R^3$ is given by 
\begin{align*}
     \fr{\pa \bfe_i^\top M^{(j)} \bfe_l}{\pa \bfx}&     =  -  \Phi'( d( \bfq^{(j)}, {\mathcal F})) \bfe_i^\top R(\theta ) \Gamma^{-1} R^\top (\theta )\bfe_l \fr{\pa d }{\pa \bfq} \fr{\pa \bfq }{\pa \bfx} \notag\\
    &\hspace{-5mm} +  \left( 1 -  \Phi( d( \bfq^{(j)}, {\mathcal F})) \right) \bfe_i^\top D(R'(\theta) \Gamma^{-1} R^\top (\theta)) \bfe_l \bfe_3^\top , \\
    \fr{\pa \bfq}{\pa \bfx} &= R'^\top (\theta) (\bfy^{(j)} - Q \bfx) \bfe_3^\top - R^\top (\theta) Q ,   
\end{align*}
where $D()$ is an operator defined by $D(A) = A + A^\top $, for any square matrix $A \in \R^{n \times n}$ with a positive integer $n \in \bbN$. 

\end{proposition}

\begin{proof}
Riccati update \eqref{eq:Riccati-diag} for the covariance matrix can be rewritten with respect to the vector states $\bfsigma^{(j)}_k, \bfm^{(j)}$ as $$
    \bfsigma^{(j)}_{k+1} = \bar \bfg(\bfsigma^{(j)}_{k}, \bfm^{(j)}(\bfx_{k+1})), $$
where $\bar   \bfg: \R^3 \times \R^3 \to \R^3$ as a function of $\bfsigma = [\sigma_1, \sigma_2, \sigma_3] \in \R^3$ and $\bfm = [m_1, m_2, m_3] \in \R^3$ is given by 
\begin{align*} 
  \bar   \bfg(\bfsigma, \bfm ) &= \left[ 
    \begin{array}{ccc}
    g_1(\bfsigma, \bfm) &
    g_2(\bfsigma, \bfm) &
    g_3(\bfsigma, \bfm) 
    \end{array} 
    \right] , \\
g_1(\bfsigma, \bfm) &= f(\bfsigma, \bfm)^{-1} (- \sigma_2^2 m_3 + \sigma_1 \sigma_3 m_3 + \sigma_1), \\
g_2(\bfsigma, \bfm) &= f(\bfsigma, \bfm)^{-1} (\sigma_2^2 m_2 - \sigma_1 \sigma_3 m_2 + \sigma_2), \\
g_3(\bfsigma, \bfm) &=  f(\bfsigma, \bfm)^{-1} (- \sigma_2^2 m_1 + \sigma_1 \sigma_3 m_1 + \sigma_3), \\
  f(\bfsigma, \bfm) &=    \sigma_2^2 (m_2^2 - m_1 m_3) + 2 \sigma_2 m_2 + \sigma_3 m_3 \notag\\
  & + \sigma_1 (m_1 (\sigma_3 m_3 + 1) - \sigma_3 m_2^2) + 1. 
\end{align*}
Thus, taking the derivatives of the equations above with respect to $\bfsigma$ and $\bfx$ leads to Proposition \ref{prop:jacob-riccati}. 
\end{proof}

\end{document}